\documentclass[11pt]{article}

\usepackage{chngcntr} 
\usepackage{amsmath}
\usepackage{enumerate}
\usepackage{fancyhdr}
\usepackage{color}
\usepackage{listings}
\usepackage{graphicx}
\usepackage{amssymb}
\usepackage{mathtools}

\usepackage{hyperref}
\usepackage{booktabs}

\usepackage{algorithm}
\usepackage{algorithmic}

\usepackage[round]{natbib}
\usepackage{subfig}
\input{defs.tex}

\definecolor{lightgray}{gray}{0.5}

\usepackage[papersize={8.5in,11in},top=1in,bottom=0.75in,left=0.75in,right=0.75in]{geometry}
\usepackage[colorinlistoftodos]{todonotes}
\usepackage{enumitem}
\usepackage{tablefootnote}

\newtheorem{theorem}{Theorem}[section]

\newtheorem{lemma}[theorem]{Lemma}
\newtheorem{corollary}[theorem]{Corollary}
\newtheorem{definition}[theorem]{Definition}
\newtheorem{assumption}[theorem]{Assumption}
\newtheorem{remark}[theorem]{Remark}

\title{\textbf{On Learning Mixture of Linear Regressions in the Non-Realizable Setting}}

\author{Avishek Ghosh$^{\dagger}$, Arya Mazumdar$^{\dagger}$, Soumyabrata Pal$^{\star}$ and Rajat Sen$^{\ddagger}$  \vspace{2mm} \\
Halıcıoğlu Data Science Institute (HDSI), UC San Diego$^\dagger$ \\
\vspace{1.5mm}
Google Research, India$^\star$  \\
Google Research, USA$^\ddagger$  \\
email: \{a2ghosh, arya\}$@$ucsd.edu, \{soumyabrata, senrajat\}$@$google.com
}

\date{}
\begin{document}
\maketitle

\begin{abstract}
While mixture of linear regressions (MLR) is a well-studied topic, prior works usually do not  analyze such  models for prediction error.  In fact, {\em prediction} and {\em loss} are not well-defined in the context of mixtures. In this paper, first we show that MLR can be used for prediction where instead of predicting a label, the model predicts a list of values (also known as {\em list-decoding}). The list size is equal to the number of components in the mixture, and the loss function is defined to be minimum among the losses resulted by all the component models. We show that with this definition, a solution of the empirical risk minimization (ERM) achieves small probability of prediction error. This begs for an algorithm to minimize the empirical risk for MLR, which is known to be computationally hard. 
Prior algorithmic works in MLR focus on the {\em realizable} setting, i.e., recovery of parameters when data is probabilistically generated by a mixed linear (noisy) model. In this paper we show that a version of the popular alternating minimization (AM) algorithm finds the best fit lines in a dataset even when a realizable model is not assumed, under some regularity conditions on the dataset and the initial points, and thereby provides a solution for the ERM. 
We further provide an algorithm that runs in polynomial time in the number of datapoints, and recovers a good approximation of the best fit lines. The two algorithms are experimentally compared.
\end{abstract}

\section{Introduction}
\label{sec:intro}

For unsupervised learning, a standard way to represent the presence of sub-populations within an overall population is the use of mixture models. In many important applications ranging from financial models to genetics, a mixture
model is used to fit the heterogeneous data. Mixture models, in particular mixture of Gaussian distributions, have been extensively studied from both the perspective of statistical inference and algorithmic approaches~\citep{dasgupta1999learning,achlioptas2005spectral,kalai2010efficiently,belkin2010polynomial,arora2001learning,moitra2010settling,feldman2008learning,chan2014efficient,acharya2017sample,hopkins2018mixture,diakonikolas2018list,kothari2018robust,hardt2015tight}.

A relatively recent use of mixtures is to model functional relationships (as in supervised learning).
 Most prominent among these is the {\em mixed linear regression} problem~\citep{de1989mixtures}. In this setting, each sample is a tuple of (covariates, label). The label is stochastically generated by picking a linear relation uniformly from a set of two or more linear functions, evaluating this function on the covariates and possibly adding noise. The goal is to learn the set of  unknown linear functions. 
The problem has some recent attention~\citep{chaganty2013spectral,faria2010fitting,stadler2010l,li2018learning,kwon2018global,viele2002modeling,yi2014alternating,yi2016solving,balakrishnan2017statistical,klusowski2019estimating}, with an emphasis on the Expectation Maximization (EM) algorithm and other alternating minimization (AM) techniques. 

However, it is to be noted, all of these previous works on mixed-linear regression and their algorithmic guarantees, assumes a generative model for the data, and then provides sample-complexity guarantees for parameter-estimation. On the other hand, in the non-realizable setting, where one is just provided with a set of (covariates, label) pairs, algorithmic methods to find two or more optimal lines were never explored. This is in sharp contrast with high-dimensional Gaussian mixtures, where in the non-realizable setting, the $k$-means, $k$-median, and similar objective functions are incredibly popular, and the algorithmic guarantees focus on approximating the objective of cost function (lack of generative modeling implies no parameter estimation to be done). For example, the famous $k$-means++ algorithm provides an approximate answer with a cost within $\log k$-factor of the optimal value~\citep{vassilvitskii2006k}.  

This begs the question, in the non-realizable data-dependent setting what is a loss/risk function that one should aim to minimize to find the best-fitting lines?  Given a set of lines, our objective should be to assign a data-sample to the line that best explains the relation between the covariates and label in the data sample. The general loss function provided in Eq.~\eqref{eq:minloss} below captures this fact, which is the minimum of some base loss function evaluated for all the lines - in short the min-loss. On the other hand, the lines should be chosen such that the total error of all the data-samples due to such assignment is minimized.  A special case would be that motivated from the $k$-means objective where we resort to a squared error metric. This is provided in Eq.~\eqref{eq:loss}), and appeared in \citep{yi2014alternating}. In fact, it was shown in \citep{yi2014alternating} that optimizing this loss function is computationally hard.

The definition of the min-loss also motivates the question whether we can use the learned model for prediction using this metric? Indeed, mixture models are mainly studied for statistical inference; it is not clear how one can use a learned mixture of linear regressions for predicting the label, given the covariates. We claim that this can be done by predicting a {\em list of labels}, instead of predicting a single label for a given set of covariates. This notion of predicting a list is not unusual: for robust linear regression this has been proposed in \citep{karmalkar2019list}. However, in our setting this is a more natural solution, as each element on the list correspond to the label produced by each component model in the mixture. While this is intuitively satisfying, the min-loss and the list-decoding paradigm still needs validation in a learning theoretic sense.

In this paper, we show that indeed it is  possible to obtain generalization bounds in the sense of PAC learning functional mixture models, using the min-loss and list decoding. That further motivates the algorithmic question of finding optimal solution for the empirical risk minimization (ERM) problem. Since the ERM for this problem is NP-hard, we propose two approximation algorithms. It is shown that the algorithms provide provable guarantees under the non-realizable setting in terms of converging to either the optimal parameters, or approximating the risk function. 



%

\subsection{Summary of Contributions}

While the main technical contribution of this paper is analysis of two algorithms in the non-realizable data dependent setting, we believe the proposition of using mixture models for prediction tasks can have potential application in variety of use-cases where it is allowable to predict a list of potential outcomes, instead of a unique answer. The application domains can be recommender systems~\citep{resnick1997recommender} or text auto-completion~\citep{kannan2016smart}, where making multiple suggestions to a user is common. Thus a learning-theoretic validation of the list-decoding is useful. Our results are summarized below.    

\noindent{\bf Generalization for min-loss:} 
We show that, under our definitions of list-prediction and min-loss, the optimizer of the ERM problem is indeed a good population predictor, by proving that the mixture-function class with min-loss has a Rademacher complexity not more than $O(k)$ times than that of the base function class, where $k$ is the number of mixture components. This provides validation of our use of min-loss with the list prediction, and also motivates approximation algorithms for the ERM. (Section \ref{sec:gen}).

\noindent{\bf An AM algorithm in the non-realizable setting:} 
We propose a version of alternating minimization (AM) algorithm, that converges to the optimal models (lines that minimizes the ERM), if initialized close enough. While the closeness condition can look stringent, in practice the algorithm can be made to converge with few restarts (different choices of initial points). The main technical difficulty in analyzing this algorithm comes from the non generative aspect of data, and we characterize a natural \emph{bias} in this framework and show that, AM converges exponentially provided this bias is (reasonably) small. Additionally, in the non-generative framework, we define a (geometric) data dependent gap indicating a natural complexity measure of the mixture problem, and characterize our results as a function of it. We emphasize that in previous works on AM \citep{yi2014alternating,yi2016solving} focused on the generative model in the noise less setting, and hence did not have any bias or data dependent gap. Furthermore, our AM algorithm converges without the requirement of re-sampling fresh data at each iterations, in contrast to the previous works on AM \citep{jain2013low,netrapalli2015phase,yi2014alternating,yi2016solving}. (Section \ref{sec:am}).

\noindent{\bf A sub-sampling based algorithm:} Instead of trying to find the optimal set of lines that minimizes the ERM, we can alternatively seek to find some set of lines that provides a value of the loss function close to the minimum. In this vein, we proposed a sampling-based algorithm that fits lines to all possible partitions of a smaller random subset of the dataset. Since the subset is smaller, this can be done exhaustively. The found lines can then be verified against the entirety of the dataset to find the best choice. Note that, the guarantee this algorithm can provide is very different from the AM algorithm. Specifically, under some mild assumptions, the algorithm can provide a $O(1/\sqrt{\log n})$ additive approximation to objective~\eqref{eq:erm} where $n$ is the number of observed examples. The run-time is polynomial in $n$ but exponential in $d$ (the covariate dimension), provided that the number of mixtures is a constant. In practice, by modifying the algorithm the dependence on $d$ can be made polynomial, but as of now we do not have a theoretical analysis for that version of the algorithm. (Section \ref{sec:subset}).

\noindent{\bf Practical considerations: experiments:} We also conduct experiments on both non-linear synthetic data and real data using the different algorithms that we propose. We reach some interesting conclusions from our empirical study, namely the following: 1) The AM algorithm is very practical in terms of running time but its performance crucially depends on its initialization 2) A minor modification of the sub-sampling based algorithm has good empirical performance both in terms of generalization error and time complexity making it a viable alternative to AM algorithm in real world applications, and 3) The modified sub-sampling based algorithm can serve as a good initialization for the AM algorithm. (Section \ref{sec:exp}).

\subsection{Related Work}
As mentioned all prior works on mixture of linear regressions are in the realizable setting, where the aim is statistical inference. Most papers aim to to do parameter estimation with near-optimal sample complexity, with a large fraction focusing on the performance of the ubiquitous expectation maximization (EM) algorithm.

Notably, in \citep{balakrishnan2017statistical}, it was shown that the EM algorithm is able to find the correct lines if initialized with close-enough estimates. Furthermore, in the finite sample setting, \citep{balakrishnan2017statistical} shows convergence within an $\ell_2$ norm ball of the actual parameters, and \citep{klusowski2019estimating} then extends in to an appropriately defined cone. In \citep{yi2014alternating}, the initial estimates were chosen by the spectral method to obtain nearly optimal sample complexity for 2 lines, and then \citep{yi2016solving} extends this to $k$ lines. Interestingly, for the special case of 2 lines, \citep{kwon2018global} shows that the any random initialization suffices.  The above works assume the covariates to be standard Gaussian random vectors. Finally, in \citep{li2018learning}, the assumption of standard Gaussian covariates is relaxed (to allow Gaussians with different covariances) and near-optimal sample and computational complexity is achieved, albeit not via the EM algorithm. Another line of research focuses on understanding the convergence rate of AM, and in \citep{ghosh2020alternating,shen2019iterative} it is shown that AM, or its variants can converge at a double exponential (super-linear) rate.

In another line of work in the realizable setting, one is allowed to design covariates to query for corresponding labels~\citep{yin2018learning,kris2019sampling,mazumdar2020recovery}. However, none of these works is directly comparable to our setting.

\section{Problem Formulation}
\label{sec:psetting}

We are interested in learning a mixture of functions from $\cX \rightarrow \cY$ for $\cX \subseteq \mathbb{R}^{d}$, that best fits a data distribution $\cD$ over $(\cX, \cY)$. The learner is given access to $n$ samples $\{x_i, y_i\}_{i=1}^n$ from the distribution $\cD$. As in usual PAC learning there exists a base function class $\cH: \cX \rightarrow \cY$ and the individual functions of the learned mixture should belong to $\cH$. However, we will work in the paradigm of list decoding where the learner is allowed to output a list of responses given a test sample $x$ each of which corresponds to a mixture component function applied to $x$. We now formally define a list-decodable function class.

\begin{definition}
For a base function class $\cH$, we can construct a $k$-list-decodable vector valued function class, denoted by $\bar{\cH}_k$ such that any $\bar{h} \in \bar{\cH}_k$ is defined as \[
\bar{h} = (h_1(\cdot), \cdots, h_k(\cdot))
\]
for some set of $h_i$'s such that $h_i \in \cH$ for all $i$. Thus $\bar{h}$'s map $\cX \rightarrow \cY^{k}$ and form the new function class $\bar{\cH}_k$. 
\end{definition}
We will omit the $k$ in $\bar{\cH}$ when clear from context. As in PAC learning, we need to define a loss measure to quantify the quality of learning. For the list decodable setting, it is natural to compete on the minimum loss achieved by any of the values in the list for a particular example~\citep{kothari2018robust}. In order to formally define the min-loss setting, let $\ell: \cY \times \cY \rightarrow \reals^+$ be a base loss function. Then the min-loss is defined as follows,
\begin{align*}
    \cL(y, \bar{h}(x)) &:= \min_{j \in [k]} \ell(y, \bar{h}(x)_j) = \min_{j \in [k]} \ell(y, h_j(x)) \\
    L(\bar{h}) &:= \frac{1}{n} \sum_{i=1}^n \cL(y_i, \bar{h}(x_i)). \numberthis \label{eq:minloss}
\end{align*}

In much of this paper we will specialize to a setting where the base function class is linear i.e $\cH = \{\inner{\theta}{\cdot}: \forall \theta \in \reals^{d} \text{ s.t } \norm{\theta}_2 \leq w\}$. In this case, we can follow a simplified notation for the min-loss,
\begin{align}\label{eq:loss}
    L(\theta_1,\ldots,\theta_k) = \frac{1}{n}\sum_{i=1}^n  \min_{j \in [k]} \left \lbrace (y_i - \inprod{x_i}{\theta_j})^2  \right \rbrace.
\end{align}
\vspace{-3mm}
\begin{align}
\label{eq:erm}
\text{with} \,\,\, \, (\theta_1^*, \ldots, \theta_k^*) = \argmin_{\{\theta_j\}_{j=1}^k} L(\theta_1,\ldots,\theta_k).
\end{align}
Our focus in this paper is on generalization of the above learning problem. In Section~\ref{sec:gen} we analyze the Rademacher complexity~\citep{mohri2018foundations} of learning the mixture function class with respect to the min-loss. However, it is known that the Emprirical Risk Minimization (ERM) problem in Eq.~\ref{eq:erm} even for the linear setting is NP-Hard~\citep{yi2014alternating}. Therefore, we propose two algorithms to (approximately) solve this ERM problems in Sections~\ref{sec:am} and~\ref{sec:subset} under different sets of assumptions, even when the dataset does not follow the correct generative model.

\section{Generalization guarantees for Supervised Learning}
\label{sec:gen}

Our main result in this section is that the Rademacher complexity~\citep{mohri2018foundations} of learning a mixture of $k$ functions for the min-loss defined above is bounded by $k$ times the Rademacher complexity of the base function class from which the mixture components are drawn. We assume that the base loss function $\ell$ in Section~\ref{sec:psetting} is $\mu$-Lipschitz i.e $|\ell(y, y_1) - \ell(y, y_2)| \leq \mu |y_1 - y_2|$ for all $y_1, y_2 \in \cY$. For the sake of completeness recall that the empirical Rademacher complexity of $\cH$~\citep{mohri2018foundations} is defined as,
\begin{align}
    \rad_{S_x}(\cH) = \frac{1}{n}\EE_{\bsigma}\left[\sup_{h \in \cH} \sum_{i=1}^{n} \sigma_i h(x_i)\right],
\end{align}
where $\bsigma$ is a set of Rademacher RV's. As a notational convenience, we use $\cD_x$ to denote the marginal distribution over the covariates and $S_x$ as the set of observed covariates. The Rademacher complexity is then defined as $\rad_n(\cH) = \EE_{S_x \sim \cD_x^n}[\rad_S(\cH)]$. We are interested in the (empirical) Radmacher complexity of the mixture function class along with the min-loss which is defined as 

\[\rad_S(\bar{\cH}_k) = \frac{1}{n} \EE_{\bsigma} \left[ \sup_{\bar{h} \in \bar{\cH}_k}\sum_{i=1}^{n} \sigma_i \cL(y_i, \bar{h}(x_i))\right]\]

where $\{x_i, y_i\}$ are $n$ samples from $\cD^n$. Our first result is the following theorem.

\begin{theorem}
\label{thm:gen}
We have the following bound,
\begin{equation}
   \rad_S(\bar{\cH}_k) \leq k\mu\rad_{S_{x}}(\cH)
\end{equation}
for any $S = \{x_i, y_i\}_{i=1}^{n}$ and $S_x = \{x_i\}_{i=1}^{n}$. Therefore, it also follows that $\rad_n(\bar{\cH}_k) \leq k\mu\rad_{n}(\cH)$.
\end{theorem}

We provide the proof in Appendix~\ref{sec:proof_gen}. In the rest of the paper our focus will be on the setting where the base function class is linear i.e $\cH = \{\inner{\theta}{\cdot}: \forall \theta \in \reals^{d} \text{ s.t } \norm{\theta}_2 \leq w\}$. It is well-known that when $\cX \subseteq \{x \in \reals^d: \norm{x}_2 \leq R\}$ that $\rad_{n}(\cH) = O\left(wR/\sqrt{n}\right)$~\citep[Chapter 11]{mohri2018foundations}. Therefore, our above result would immediately yield that the Rademacher complexity of the mixture function class with min-loss satisfies $\rad_n(\bar{\cH}_k) = O\left((k\mu wR)/\sqrt{n}\right)$.

\section{Algorithms in the non-realizable setting}

\subsection{AM Algorithm}
\label{sec:am}
Recall that our goal is to address the problem of mixed linear regression, without any generative assumption. Given the data-label pairs $\{x_i,y_i\}_{i=1}^n$, where $x_i \in \reals^d$ and $y_i \in \reals$, our goal is to fit $k$  linear predictors, namely $\{\theta^*_j\}_{j=1}^k$. We analyze an alternating minimization (AM) algorithm, and show that, provided \emph{suitable} initialization, AM converges at an exponential speed.

We first fix a few notation here. Suppose that the optimal parameters $\{\theta^*_j\}_{j=1}^k$ partition the dataset $\{x_i,y_i\}_{i=1}^n$ in $k$ sets, $\{S^*_j\}_{j=1}^k$, where 
$$S^*_j = \{ i \in [n]: (y_i - \inprod{x_i}{\theta^*_1})^2  = \min_{j \in [k]} (y_i - \inprod{x_i}{\theta^*_j})^2\},$$ 
and similarly for $S^*_2,\ldots,S^*_k$. In words, $S^*_1$ is the set of observations, where $\theta^*_1$ is a better (linear) predictor compared to $\theta^*_2,\ldots,\theta^*_k$.

\subsubsection{Gradient Alternating Minimization (AM)}
We now propose a provable algorithm in this section. The algorithm to find the best linear predictors in mixed setup is the classical Alternating Minimization (AM) \citep{yi2014alternating,yi2016solving}. Our proposed approach in formally given in Algorithm~\ref{alg:am}. Every iteration of the algorithm has in two steps: (a)  given data, we first estimate which, out of $k$ predictors obtains minimum loss in fitting the data (through a quadratic loss function); and (b) after this, we take a gradient step with a chosen step size to update the underlying parameter estimates.

In the first step, based on the current estimates, i.e., $\{\theta^{(t)}_j\}_{j=1}^k$, Algorithm~\ref{alg:am} first finds the set of indices, $S_j^{(t)}$ for $j \in [k]$, and partitions the observations. We first calculate the residual on each of the possible predictors and choose the one minimizing it, and thereafter a partition is formed by by simply collecting the observations corresponding to a particular predictor. In this way, we form $S^{(t)}_j$ for all $j \in [k]$ at iteration $t$.

In the second step, we take one iteration of gradient descent over the losses in each partition $S^{(t)}_j$, with an appropriately chosen step size. The loss corresponding to each observation in partition $S^{(t)}_j$ is the quadratic loss. Note that unlike classical AM, we do not obtain the next iterate by solving a linear regression problem---instead we take a gradient step with respect to an appropriately defined loss function. This is done for theoretical tractability of Algorithm~\ref{alg:am}.

\noindent{\bf Estimation without-resampling at each iteration:} One attractive feature of our algorithm is that, unlike classical AM for mixtures of regression, our algorithm does not require any re-sampling (fresh samples for every iteration). Note that AM based algorithms typically require resampling to nullify the independence issues across iterations, and simplifies the analysis. \citep{netrapalli2015phase} uses it for phase retrieval, \citep{jain2013low} for low-rank matrix completion, \citep{ghosh2020efficient} for distributed clustering and \citep{yi2014alternating,yi2016solving} uses it for mixed linear regression.

Note that, since our approach is entirely \emph{data-driven}, the optimal parameters $\{\theta^*_j\}_{j=1}^k$ actually depend on the entire data $\{x_i,y_i\}_{i=1}^n$. Hence, re-sampling would not make sense in the framework. In the sequel, we characterize the cost of not having the flexibility of resampling in our non-realizable setting. In particular we show that we require a stronger initialization condition compared to that of \citep{yi2014alternating,yi2016solving}, where the authors use the generative framework as well as re-sampling. 

\begin{algorithm}[t!]
  \caption{Gradient AM for Mixture of Regressions}
  \begin{algorithmic}[1]
 \STATE  \textbf{Input:} $\{x_i,y_i\}_{i=1}^n$, Step size $\gamma$ 
 \STATE \textbf{Initialization:} Initial iterate $\{\theta^{(0)}_j\}_{j=1}^k$  \\
  \FOR{$t=0,1, \ldots, T-1 $}
\STATE \underline{Partition:}  \\
 \STATE Construct $\{S_j^{(t)}\}_{j=1}^k$ such that
 \vspace{-2mm}
 \begin{align*}
     S^{(t)}_j &= \{ i \in [n]: (y_i - \inprod{x_i}{\theta^{(t)}_j})^2 \\
     & \quad = \min_{j' \in [k]} (y_i - \inprod{x_i}{\theta^{(t)}_{j'}})^2\} \,\,\, \forall \,\, j \in [k]
 \end{align*}
\STATE \underline{Gradient Step:}
\vspace{-1mm}
\begin{align*}
    \theta^{(t+1)}_j = \theta^{(t)}_j - \frac{\gamma}{n}\sum_{i \in S^{(t)}_j} \nabla F_i(\theta^{(t)}_j),  \,\,\, \forall \,\, j \in [k]
\end{align*}
\vspace{-2mm}
  \STATE where $F_i(\theta^{(t)}_j) = (y_i - \inprod{x_i}{\theta^{(t)}_j})^2$
  \ENDFOR
  \STATE \textbf{Output:} $\{\theta^{(T)}_j\}_{j=1}^k$
 \end{algorithmic}
  \label{alg:am}
\end{algorithm}

\subsection{Theoretical Guarantees for Gradient AM}
\label{sec:guarantees}
In this section, we obtain theoretical guarantees for Algorithm~\ref{alg:am}. Recall that we have $n$ samples $\{x_i,y_i\}_{i=1}^n$, where $x_i \in \mathbb{R}^d$, and we do not assume any generative model for $y_i$. Furthermore, we want to fit $k$ lines to the data---in other words, we want to estimate $\{\theta^*_j\}_{j=1}^k$.

Without loss of generality, and for scaling purposes, we take
\begin{align*}
    \max_{i \in [n]} \|x_i\| \leq 1 \,\,\, \text{and} \,\,\, \max_{i\in [n]}|y_i| \leq 1.
\end{align*}
Note that the upper-bound of $1$ can indeed be scaled by the maximum norm of the data. 

Recall that $\{S^*_j\}_{j=1}^n$ are the true partition given by the optimal parameters $\{\theta^*_j\}_{j=1}^n$. Before providing the theoretical result, let us define a few problem dependent geometric parameters. We define the (problem dependent) separation as
\begin{align*}
    \Delta = \min_{j \in [k]} \,\, \min_{i \notin S^*_j} | y_i - \langle x_i,\theta^*_j \rangle|.
\end{align*}
Note that, under the generative model, $\Delta \leq \min_{j \neq \ell} \|\theta^*_\ell - \theta^*_j\|$, which is the usual definition of separation, as proposed in \citep{yi2016solving}. The above is a natural extension to the non-generative framework.

We make the following assumption:
\begin{assumption}
\label{asm:bounded}
For all $i \in S^*_j$, we have $\max_{i} |y_i - \langle x_i, \theta^*_j \rangle| \leq \lambda$ and $ \max_{i} \|\nabla F_i(\theta^*_j)\| \leq \mu$, where $F_i(\theta) = (y_i - \langle x_i , \theta \rangle)^2$. Furthermore, $(d,k,\Delta,\lambda)$ satisfies $k\exp \left( - C \frac{(\Delta-2\lambda)^2}{(\max_{j \in [k]} \|\theta^*_j\|)^2}\, d \right) \leq 1- c$, where $c< 1$ is a constant.
\end{assumption}
Let us explain the above assumption in detail. In particular, if $y_i$ has a generative model (without noise), $\lambda=0 , \mu = 0$. This is because, for $i \in S^*_j$, we have $y_i = \langle x_i, \theta^*_j \rangle$, and the above term vanishes. The previous works on mixed linear regression, for example \citep{yi2014alternating,yi2016solving,ghosh2020alternating}, analyzes this setup exactly. A generative model is assumed and the analysis is done in the noise less case. However, we do not assume any generative model here and hence we need to control these bias parameters. Shortly, we show that, provided $\mu$ is sufficiently small, Algorithm~\ref{alg:am} converges at an exponential speed.

\paragraph{Interpretation of $\mu$ and $\lambda$:} We emphasize that $\lambda$ is (a bound) related to the variance of the data generating process. As we will se subsequently, we can tolerate reasonably  large $\lambda$ as it only appears inside a decaying exponential. It intuitively implies that we can tolerate outlier data as long as its variance is $\lambda$ bounded. Furthermore, $\mu$ is bounding the gradient norm of the loss function for outlier data. $\mu$ is large, if data points are far away from the regressor, $\theta^*_i$ (an hence an outlier). We can tolerate the outliers upto a gradient norm of $\mu$. This is also connected to the variance of the data generating process. 

Finally, the condition on $(d,\Delta,\lambda)$  is quite mild, since the left hand side is an exponential decaying function with dimension $d$. Hence, for moderate dimension, this mild condition is easily satisfied. Furthermore, it decays exponentially with the gap $\Delta$ as well, and so with a separable problem, the above condition holds.

Furthermore, to avoid degeneracy, i.e., $\{|S^*_j|\}_{j=1}^k$ is small, we have, for all $i \in [n]$, $\min_{j\in [k]}\PP[i \in S^*_j] \geq \Bar{c}$ where $\Bar{c}$ is a constant. Note that these type of assumption is quite common, and used in previous literature as well. As an instance, \citep{yi2014alternating,yi2016solving,ghosh2020alternating} use it for mixture of regressions, \citep{ghosh2020efficient} use it for AM in a distributed clustering problem, and \citep{ghosh2019max} use it for max-affine regression.

We now present our main results. In the following, we consider one iteration of Algorithm~\ref{alg:am}, and show a contraction in parameter space.
\begin{theorem}
\label{thm:am}
Let Assumption~\ref{asm:bounded} holds, $n \geq C \,k\, d$ and the covariates $\{x_i\}$ have zero mean. Furthermore, suppose
\begin{align*}
    \|\theta_j - \theta^*_j \| \leq \frac{\Tilde{c}}{\sqrt{d}}\,\, \|\theta^*_j\|.
\end{align*}
for all $j \in [k]$. Then, running one iteration of Algorithm~\ref{alg:am} with $\gamma = 1/4\Bar{c}$, yields $\{\theta^{+}_j\}_{j=1}^k$ satisfying
\begin{align*}
    \|\theta^+_j - \theta^*_j\| & \leq \frac{1}{2} \|\theta_j - \theta^*_j\| + c_1 \, \, \mu \, \|\theta^*_j\| \\
    & \qquad + c \,k\, \exp \left( - c_1 \frac{(\Delta-2\lambda)^2}{(\max_j \|\theta^*_j\| )^2}\, d \right) \|\theta^*_j\|,
\end{align*}
with probability at least $1-C_1 \exp(-C_2 d)$.
\end{theorem}
\begin{remark}[Initialization]
The $1/\sqrt{d}$ factor in the initialization occurs owing to a covering net argument. We emphasize that in the standard AM with realizable setting, the regressors $\{\theta^*_i\}_{i=1}^k$ are fixed vectors, but in our non-realizable setting, they are dependent on the dataset $\{x_i,y_i\}_{i=1}^n$. Given this, techniques that yield dimension independent initialization like \emph{resampling in each iteration}, \emph{Leave One Out (LOO)} can not be applied. We overcome this challenge by using an $\epsilon$-net, and as a result, we have the dimension dependence. As of now, we are unaware of any technical tools that remove the dimension dependence, and we hypothesize that it may be hard (or even impossible) to remove this condition theoretically. However, as we show in experiments, this condition can be relaxed.
\end{remark}
\begin{corollary}
Suppose $\mu$ satisfies
\begin{align*}
    \mu \leq c_2 \,\, \left( \frac{\Tilde{c}}{\sqrt{d}} - c \,k\, \exp \left( - c_1 \frac{(\Delta-2\lambda)^2}{(\max_j \|\theta^*_j\| )^2}\, d \right) \right),
\end{align*}
where $c_2$ is an universal constant. With proper choice of $c_2$, we obtain, for all $j \in [k]$,
\begin{align*}
    \|\theta^+_j - \theta^*_j\| \leq \frac{3}{4} \|\theta_j - \theta^*_j\|,
\end{align*}
which implies a contraction.
\end{corollary}
\begin{remark}
Note that the above implies that after sufficiently many iterations, the iterate of Algorithm~\ref{alg:am} converges exactly to $\{\theta^*_j\}_{j=1}^k$ (exact parameter recovery). This is in contrast to existing works  \citep{balakrishnan2017statistical,klusowski2019estimating} where, in the finite sample noisy setting, the convergence is guaranteed on a set close to the optimal parameters. Note that our non-generative framework may accommodate noise in the observations as well.
\end{remark}

\textit{Proof Sketch:}
We first show that by the virtue of the initialization condition and the separation, the partitions produced by Algorithm~\ref{alg:am} are  \emph{reasonable} good estimates of the true partitions $\{S^*_j\}_{j=1}^k$. We crucially exploit the bounded-ness of covariates along with a discretization (covering) argument to obtain this. Then, we run gradient descent on the estimated partitions, and use the bias $(\mu,\lambda)$ to characterize its performance. The linear convergence rate comes from the analysis of gradient descent on the quadratic loss, and the additional errors originate from (a) biases in the system and (b) estimation error in the partitions.
\subsection{Sub-Sample Algorithm}
\label{sec:subset}
We now provide an approximation algorithm to the objective in Section~\ref{sec:psetting} which requires no initialization conditions like Algorithm~\ref{alg:am}. It has been established in~\citep{yi2014alternating} that the problem is NP-hard without any assumptions and \citep{yi2014alternating} only show recovery guarantees in the realizable setting assuming isotropic Gaussian covariates. Our sub-sample based Algorithm~\ref{algo:subset} works without such assumptions on the covariates and provides an additive approximation of $\tilde{O}(1/\sqrt{\log n})$ with a runtime of $\tilde{O}(n^{c\log k} \exp(k^2d \log k))$, with high probability.

\begin{algorithm}[htbp]
\caption{Sub-sampled data driven learning of mixture of $k$ regressions}
\label{algo:subset}
\begin{algorithmic}[1]
\STATE Select a set of $A$ of data-points uniformly at random with replacement from the dataset ($|A|$ specified later).
\FOR {each partition $\{P_i\}_{i=1}^{k}$ of $A$}
\STATE Let $\theta_i := \thetals(P_i)$ for all $i$.
\STATE Use $\theta_i$'s to partition all the samples into $S = \cup_{i=1}^{k} S_{\theta_i}(S)$.
\STATE Let $\cE(\{P_i\}) =  \frac{1}{n} \sum_{i=1}^{k} \cE_{ls}(S_{\theta_i}(S), \theta_i)$.
\ENDFOR
\STATE Find $\theta_i$'s corresponding to $\{P^m_i\} = \argmin_{\{P_i\}} \cE(\{P_i\})$. Return  $\{\theta_i^+\}$ such that $\forall i$, $\theta_i^+ = \thetals(S_{\theta_i}(S))$.
\end{algorithmic}
\end{algorithm}

Before we proceed we need some notation specific to Algorithm~\ref{algo:subset}. Recall that we have $n$ examples $S = \{x_i, y_i\}_{i=1}^n$ in total. For a set of samples $P \subseteq S$, $\thetals(P)$ refers to the least-squares solution over that set i.e \[\thetals(P) = \argmin_{\theta} \sum_{i \in P} (y_i - \inner{\theta}{x_i})^2.\] Further, $\cE_{ls}(P, \theta) = \sum_{i\in P} (y_i - \inner{\theta}{x_i})^2$ and \[S_{\theta_l}(P) = \{i \in P: \text{ s.t } l = \argmin_{j \in [k]} (y_i - \inner{\theta_j}{x_i})^2 \}\] i.e the set of data points in $P$ for which the $\theta_l$ is the best fit among $\{\theta_j\}_{j=1}^{k}$. 

Let $\{B^*_i\}_{i=1}^{k}$ be the optimal partition of the samples in $S$ i.e each partition has a different $\theta^*_i$ as the solution for least-square and together they minimize the objective in Section~\ref{sec:psetting}. For any $B \subseteq S$, let us define $\Sigma_B = (1/ |B|)\sum_{i \in B} x_ix_i^T$ and $\theta_{B} = \Sigma_B^{-1} (1/|B|  \sum_{i \in B} x_iy_i)$. The theoretical guarantees on Algorithm~\ref{algo:subset} are provided under the following mild assumptions.

\begin{assumption}
\label{assum:subset}
The dataset $A$ satisfies the following: 
\begin{itemize}[noitemsep]
    \item Let $\alpha_i := |B_i^*|/n$ and we have that $\alpha_i \geq \alpha >0$ for all $i \in [k]$. 
    \item  The minimum covariance matrix eigen-values of the partitions, $\lmin(\Sigma_{B^*_i}) > \lmin > 0$ for all $i \in [k]$. 
    \item $\norm{\Sigma_{B^*_i}^{-1/2} x_j (y_j - \theta_{j}^*.x_j)} \leq \gamma \sqrt{d}$ almost surely, for all $j \in B^*_i$. Note that this $\gamma$ just needs to be finite and only appears in lower order terms.
    \item $\frac{1}{|B^*_i|}\sum_{j \in B^*_i} (y_j - \theta_{j}^*.x_j)^2 \leq B_d$ i.e the expected squared bias is bounded in each part.
    \item The response and the covariates satisfy $|y_i| \leq b$, $\norm{x_i} \leq R$. We also assume that $\theta^*_i$'s are s.t. $\norm{\theta_i^*} \leq w$.
\end{itemize}
\end{assumption}

The first assumption just ensures that none of the parts of the optimal partitions are $o(n)$. The rest of the assumptions are related to assumptions of Theorem~2 in~\citep{hsu2011analysis}, a result that we crucially use in our analysis, as will be evident in our proof sketch below.

We are now ready to present our main theorem which is an approximation result for the objective in Section~\ref{sec:psetting}. 

\begin{theorem}
\label{thm:subsamp}
Under Assumption~\ref{assum:subset}, we  guarantee with probability at least $1 - 2\delta$ that the solution returned by Algorithm~\ref{algo:subset} satisfies:
$
      L(\theta_1^+, \ldots, \theta_k^+) \leq    L(\theta_1^*, \ldots, \theta_k^*) + \epsilon,
$
provided $|A| = \tilde{\Omega} \left( \frac{1}{\epsilon^2} k^2 \frac{\alpha}{\lmin} \left(d + \log \frac{k}{\delta} \right)\right)$.
\end{theorem}

\textit{Proof Sketch:} Of all the partitions of $A$, one of them $\{\tilde{P}_i\}$ would conform to the optimal partitioning ${B^*_i}$ of the whole dataset. In that case the samples from $\tilde{P}_i$ can be akin to a random design regression problem where the examples are drawn i.i.d from $B^*_i$. Therefore, with a careful application of random design regression results~\citep{hsu2011analysis}, we can argue that the error incurred by $\theta_{\tilde{P}_i^*}$ cannot be too much larger than that of $\theta^*_i$ on the samples in $B_i^*$. This will be true for all $i \in [k]$. The error incurred by the solution of Algorithm~\ref{algo:subset} is bound to be even smaller, thus concluding the argument.

The above result implies that when the number of sub-sampled datapoints in Algorithm~\ref{algo:subset} is \\
$\Omega(k^2 (\alpha/ \lmin) \log n  (d + \log(k/\delta)))$, we can guarantee an approximation of the order of $1/ \sqrt{\log n}$. Therefore, the loop in line 2 of Algorithm~\ref{algo:subset} implies a runtime of $\tilde{O}(n^{c\log k} \exp(k^2d \log k))$ where $c$ is an universal constant. Thus our time complexity is polynomial in $n$ but exponential in the dimension $d$. ~\cite{yi2014alternating} showed that the objective in~\eqref{eq:erm} is NP-Hard with respect to the dimension. We conjecture that the the problem is also hard with respect to $n$, but an analysis of the same is beyond the scope of this paper.
\section{Experiments}\label{sec:exp}

In this section, we validate our theoretical findings via experiments. 

\paragraph{A note on AM convergence:}
Note that in many supervised real world data problems with incomplete data, corruptions or latent clusters, the most natural, computationally efficient and scalable algorithm that is widely used in practice is the AM algorithm (and its soft counterpart EM). For real world data, \emph{the assumption of realizibility is naturally violated}, and as we will show subsequently (via Movielens data), that AM converges and attains small error. Furthermore, existing literature have already demonstrated the convergence and superior performance of AM type algorithms on real world data (see \citep{ghosh2020efficient,balazs2016convex} , where in \citep{ghosh2020efficient} experiments are done with MNIST, CIFAR 10 and FEMNIST data and \citep{balazs2016convex} uses aircraft profile drag  and inventory control data). All of these (non realizable) real data experiments work \emph{without any suitable initialization}. Random initialization with multiple restarts is a standard, well accepted technique for AM.

\subsection{Synthetic Dataset:} 
Note that in Step 2 of Algorithm \ref{algo:subset}, we iterate through all partitions of the set $A$; unfortunately, in practice, this step is time-expensive. However it is possible to modify this step slightly to make this algorithm anytime i.e. we can have a budget on time and recover the best possible solution within that time. Below, we propose Algorithm \ref{algo:subset2}  (that is a minor modifications of Algorithm \ref{algo:subset}) for practical applications:

\begin{algorithm}[t!]
\caption{Sub-sampled data driven learning of mixture of $k$ regressions with random partitions}
\label{algo:subset2}
\begin{algorithmic}[1]
\REQUIRE Input hyper-parameter $h$
\STATE Select a set $A$ of data-points uniformly at random with replacement from the dataset ($|A|$ specified in Thm \ref{thm:subsamp}).
\FOR {each of $h$ random partitions $\{P_i\}_{i=1}^{k}$ of $A$}
\STATE Let $\theta_i := \theta_{\mathsf{rob/ls}}(P_i)$ for all $i$.
\STATE Use $\theta_i$'s to partition all the samples into $S = \cup_{i=1}^{k} S_{\theta_i}(S)$.
\STATE Let $\cE(\{P_i\}) =  \frac{1}{n} \sum_{i=1}^{k} \cE_{rob/ls}(S_{\theta_i}(S), \theta_i)$.
\ENDFOR
\STATE Find $\theta_i$'s corresponding to $\{P^m_i\} = \argmin_{\{P_i\}} \cE(\{P_i\})$. Return  $\{\theta_i^+\}$ such that $\forall i$, $\theta_i^+ = \theta_{\mathsf{rob/ls}}(S_{\theta_i}(S))$.
\end{algorithmic}
\end{algorithm}


On the other hand, in Step 2 of Algorithm \ref{algo:subset2}, we iterate through $h$ random partitions (where $h$ is chosen according to the time budget). There are two options two fit each part of the partition of the set $A$ in Step 3 namely 1) we can fit a linear model as was the case in Algorithm \ref{algo:subset2} 2) we can fit a linear regression model that is robust to outliers in Step 3 (can be implemented using \texttt{sklearn.linearmodel.RANSACRegressor()} in python). Since random partitions will contain elements from from all sets $\{B^{*}_i\}_{i=1}^{k}$, a robust regression model is able to better fit each part of the partition.

\begin{table}[htbp]

\begin{center}
\begin{tabular}{ c|c|c|c } 
 \toprule
 Dataset & LR  & Alg \ref{alg:am} (Mean, Var) & Alg \ref{algo:subset2} (Mean,Var) \\ \midrule
   A &     20.13  &  (16.60,16.98) & (11.84,0.41)  \\
   B &     19605  & N/A & (5002.03, 50925) \\
   C &     12134  & N/A & (7.24,0.16)  \\    
 \bottomrule
\end{tabular}

\end{center}
\caption{Mean (Mean) and Variance (Var) of min-loss on Test data generated from non-linear synthetic datasets. LR corresponds to a simple linear regression model; N/A implies that the algorithm did not converge during training in any implementation.}
\label{table:nonlinear2}
\end{table}

\begin{table*}[t!]
\begin{center}
\begin{tabular}{ c|c|c|c|c|c|c } 
 \toprule
 $u$ & $v$ & A0 (Mean,Var)  & A1 (Mean,Var) & A2( Mean,Var) & A3(Mean, Var) & A4(Mean,Var,DNC)  \\ \midrule
 `1010' & `2116' & (0.709,0)  & (0.4586,0.0003) & (0.33,0.0005) & \textbf{(0.2416,0.0003)}  & (0.5949,0.07,6)  \\
 `752' & `1941' & (1.423, 0) & (0.801, 0.001) & (0.416, 0.0005) & \textbf{(0.360, 0.001)} & (0.588, 0.13, 3) \\
 `752' &  `2116' & (1.475,0) & (0.90, 0.0009) & \textbf{(0.36, 0.0023)} & (0.40, 0.003) & (0.79, 0.17, 6)  \\
 `752' & `2909' & (1.40, 0) & (0.855,0.0014) & \textbf{(0.42,0.0005)} & (0.425,0.002) & (0.96, 0.21, 5)\\
 `1010' & `4725' & (0.71, 0) & (0.4566, 0.0002) & (0.33, 0.0007) & \textbf{(0.269, 0.0003)} & (0.51, 0.04, 7)\\
 \bottomrule
\end{tabular}

\end{center}
\caption{Training error (min-loss) for the 5 different algorithms A0 (Linear Regression), A1 (Algorithm \ref{algo:subset2} with linear model), A2 (Algorithm \ref{algo:subset2} with robust linear model), A3 (Algorithm \ref{alg:am} with initialization by Algorithm \ref{algo:subset2}), A4 (Algorithm \ref{alg:am} with random initialization). For each algorithm, we report the mean (Mean) and variance (Var) of min-loss on training data over $30$ implementations. For A4, we also report the number of times (DNC) out of $30$ that the algorithm did not converge. For each row, the numbers in bold correspond to the model with the minimum average training error.}
\label{table:samples}
\end{table*}

\begin{table*}[t!]
\begin{center}
\begin{tabular}{ c|c|c|c|c|c|c } 
 \toprule
 $u$ & $v$ & A0 (Mean,Var)  & A1 (Mean,Var) & A2( Mean,Var) & A3(Mean, Var) & A4(Mean,Var)  \\ \midrule
 `1010' & `2116' & (0.62,0)  &  (0.394,0.0003) & (0.288,0.0009) &  \textbf{(0.23,0.0005)}  & (0.56,0.044)  \\
 `752' & `1941' & (1.364,0) & (0.742,0.001) & (0.399,0.0004) & \textbf{(0.3133,0.001)} & (0.6224,0.14) \\
 `752' &  `2116' & (1.29,0) & (0.82,0.0012) & \textbf{(0.30,0.0014)} & (0.406,0.008) & (0.88,0.20)  \\
 `752' & `2909' & (1.45,0 ) & (0.846,0.0025) & (0.49,0.0009) & \textbf{(0.4191,0.0029)} & (1.13,0.31)\\
 `1010' & `4725' & (0.685,0) & (0.4327,0.0004) & (0.289,0.0009) & \textbf{(0.258,0.0001)} & (0.622,0.088)\\
 \bottomrule
\end{tabular}

\end{center}
\caption{Test error (min-loss) for the 5 different algorithms A0 (Linear Regression), A1 (Algorithm \ref{algo:subset2} with linear model), A2 (Algorithm \ref{algo:subset2} with robust linear model), A3 (Algorithm \ref{alg:am} with initialization by Algorithm \ref{algo:subset2}), A4 (Algorithm \ref{alg:am} with random initialization). For each algorithm, we report the mean (Mean) and variance (Var) of min-loss on test data over $30$ implementations. For each row, the numbers in bold correspond to the model with the minimum average test error.}
\label{table:samples2}
\end{table*}
\noindent{\bf Non-linear datasets:} We implement and compare the performance of our algorithms on three non-linear datasets generated by \texttt{sklearn} namely \texttt{makefriedman1} [A]\citep{fried1}, \texttt{makefriedman2}[B] \citep{fried2} and \texttt{makefriedman3}[C] \citep{fried3}. Note that all these datasets are non-realizable.

All the three datasets A, B and C comprises of 3200 samples in the train data and 800 samples in the test data. The covariates in datasets A, B and C comprises of 5, 4 and 4 features respectively. We train and test three different algorithms on these datasets and report the average min-loss on the test set in~\ref{table:nonlinear2}. In Table \ref{table:nonlinear1} in the appendix we also report the training errors for the sake of completeness The three algorithms are the following: 1) \textit{Linear Regression (LR)} As a baseline, we fitted a linear regression model on the training dataset and reported the train/test error (Mean squared error) 2) \textit{Algorithm \ref{alg:am}}  For dataset A, we implement Algorithm \ref{alg:am} with $\gamma=0.1$ and random initialization (every element of $\theta^{(0)}_1,\theta^{(0)}_2$ is generated i.i.d according to a Gaussian with mean $0$ and standard deviation $10$). We report the average train/test min-loss over $50$ iterations. For  datasets B and C,  Algorithm \ref{alg:am} did not converge while training with the many different initializations that we used (including random initializations and initialization with the best solution provided by Algorithm \ref{algo:subset2} using the linear model). 3) \textit{Algorithm \ref{algo:subset2} with robust linear model:} For all three datasets A, B and C, this algorithm was implemented $30$ times with $h=1000$ and the average train/test min-loss is reported. Empirically, we can conclude that  the performance of AM algorithm depends crucially on the initialization; this is reflected in the high variance in min-loss for both train and test data (dataset A) and non-convergence for datasets B and C. On the other hand, Algorithm \ref{algo:subset2} (with the robust linear model) does not have the initialization drawback and has improved empirical performance. 

We provide \emph{negative} results of AM on synthetic datasets B and C because these are \emph{outlier datasets} where the AM algorithm has a poor performance and therefore demands further investigation both theoretically and empirically. 

\subsection{Real-world Data: Movielens Data}
For experiments on real datasets, we use 
the Movielens 1M dataset\footnote{https://grouplens.org/datasets/movielens/1m/} that consists of $1$ million ratings from $m=6000$ users on $n=4000$ movies. In our first pre-processing step, we generate a $4$-dimensional embedding of the users and movies in the following way. Note that we observe only a fraction of the entries of the entire user-item matrix $M$. Let the set of users be described by a matrix $U\in \mathbb{R}^{m \times 4}$ and the set of movies be described by a matrix $V\in \mathbb{R}^{n \times 4}$. In that case if $\Omega$ is the set of observed entries in $M$, then a usual way \citep{lu2013second} to estimate $U,V$ is to optimize the following objective:
$
    \min_{U,V\in \mathbb{R}^{m \times 4}} \sum_{(i,j)\in \Omega} (M_{ij}-\langle U_i, V_j \rangle)^2
$
where $U_i$ ($V_j$) denotes the $i$th ($j$th) row of $U$ ($V$). 

Subsequently, we choose an arbitrary pair of users $(u,v)$ such that both the users have rated a large number of movies and furthermore, $u$ and $v$ have significantly different preferences. We ensure this by choosing $(u,v)$ such that $u$ has given a high rating to a small fraction of movies they have rated while $v$ has given a high rating to a large fraction of movies they have rated. Next, for user $u$, we create a dataset $(X_u,y_u)$ where each row of $X_u$ corresponds to the embedding of a movie that user $u$ has rated while the corresponding element of $y_u$ is the rating for the aforementioned movie. We create a similar dataset $(X_v,y_v)$ for user $v$ and subsequently, we combine these datasets (and shuffle) to generate a single dataset $(X,y)$. We split this dataset into train and test ($80:20$); in Table \ref{table:samples} (in the appendix), we report the user ids and number of samples in train and test data.  

We compute the min-loss for five different algorithms on the train and test data (averaged over $30$ implementations) for each pair of users and report them in Tables \ref{table:samples} and \ref{table:samples2}. The algorithms that we have implemented are the following: 1) \textit{Linear Regression (A0):} As a baseline, we fitted a simple linear regression model on the training data and computed the mean squared error on the test data (which corresponds to the min loss in case of a single line) 2) \textit{Algorithm \ref{algo:subset2} with linear model (A1):} We implement Algorithm \ref{algo:subset2} with $|A|=150$ and $h=1000$ and in Steps 3,5 we use the Linear Regression model 3) \textit{Algorithm \ref{algo:subset2} with robust regression model (A2):} We implement Algorithm \ref{algo:subset2} with $|A|=150$, $h=1000$ and in Steps 3,5 we use the robust regression model 4) \textit{Algorithm \ref{alg:am} (A3)}  We implement Algorithm \ref{alg:am} initialized with the lines that is obtained after fitting Algorithm \ref{algo:subset2} (A1) on the training data 5) \textit{Algorithm \ref{alg:am} (A4)} We implement Algorithm \ref{alg:am} with random initialization $\theta_1^{(0)},\theta_2^{(0)}$ where each entry of the initialized vectors are generate i.i.d according to a Gaussian with zero mean and variance $2$.

\section{Conclusion}
It is evident from Tables \ref{table:samples} and \ref{table:samples2}   that Algorithm \ref{algo:subset2} (using the robust linear module) and Algorithm \ref{alg:am} (initialized by the best solution of Algo. \ref{algo:subset2} where we use the linear model to fit each part of the partition of $A$) has the best empirical performance overall. Note that AM (Algorithm \ref{alg:am}) often, due to bad initialization,  gets stuck to local minima.
In Tables   \ref{table:samples} and \ref{table:samples2}, it is clear from the final column (A4) that the variance of the AM algorithm (with random initialization), is very high. 
Algorithm \ref{algo:subset2} (using the linear regression module) can provide a viable alternative to spectral clustering as an initialization technique for the AM algorithm. Algorithm \ref{algo:subset2} (using the robust linear module aka A2)  has a similar performance (both in terms of train/test error and time complexity) to Algorithm \ref{alg:am}. It is a very interesting future direction to study the guarantees of A2 theoretically. Unlike Algorithm \ref{alg:am}, A2 does not suffer from the initialization issue and can be a potentially viable alternative to AM algorithm in practice.

\bibliographystyle{abbrvnat}
\bibliography{references}

\appendix
\vspace{5mm}
\newpage
\begin{center}
    \Large \textbf{Appendix}
\end{center}

\section{Proof of Theorem~\ref{thm:am}}
Without loss of generality, let us focus on $j=1$, i.e., $\thplusone$.
We have
\begin{align*}
   \| \thplusone -\theta^*_1 \| &= \| \theta_1 - \theta^*_1 -\frac{\gamma}{n}\sum_{i \in S_1}\gradone \| \\
    & = \| (\theta_1 - \theta^*_1) - \frac{\gamma}{n}\sum_{i \in S_1}(\gradone -\nabla F_i(\theta^*_1)) - \frac{\gamma}{n}\sum_{i \in S_1}\nabla F_i(\theta^*_1) \| \\
    & \leq \underbrace{ \| (\theta_1 - \theta^*_1) - \frac{\gamma}{n}\sum_{i \in S_1}(\gradone -\nabla F_i(\theta^*_1))\|}_{T_1} + \frac{\gamma}{n} \underbrace{\|\sum_{i \in S_1}\nabla F_i(\theta^*_1) \|}_{T_2}.
\end{align*}
Let us first consider $T_1$. We have
\begin{align*}
    T_1 = \| (I - \frac{2\gamma}{n}\sum_{i \in S_1} x_i x_i^\top) (\theta_1 - \theta^*_1)\|.
\end{align*}
We first observe
\begin{align*}
    \sigma_{\min}(\frac{1}{n}\sum_{i \in S_1} x_i x_i^\top ) \geq \sigma_{\min}(\frac{1}{n}\sum_{i \in S_1 \cap S^*_1} x_i x_i^\top )
\end{align*}
In order to upper bound $T_1$, we need to control the size of $S_1 \cap S^*_1$. Using Lemma~\ref{lem:error}, we argue that provided the initialization condition is satisfied, $|S_1 \cap S^*_1| \geq \frac{1}{2}|S^*_1| \geq c p_1 n$ with probability at least $1-\exp(-c_1 n)$, where $p_1$ is the fraction of observations, for which $\theta^*_1$ is a better predictor. Also Lemma~\ref{lem:subgauss} shows that conditioned on $i \in S_1 \cap S^*_1$, the random variables $x_i$ is still sub-Gaussian with constant parameter. Hence, using \cite{vershynin2018high}, we have
\begin{align*}
    \sigma_{\min}(\frac{1}{n}\sum_{i \in S_1} x_i x_i^\top ) \geq c_1 p_1 n
\end{align*}
with probability at least $1-\exp(-d)$, and as a result,
\begin{align*}
    T_1 \leq (1- c \gamma p_1) \|\theta_1 -\theta^*_1\|,
\end{align*}
with probability at least $1-\exp(-c_1n) - \exp(-d) \geq 1-c_1 \exp(- c_2 d)$.
\vspace{2mm}

\noindent Let us now consider the term $T_2$. We have
\begin{align*}
    T_2 &= \frac{\gamma}{n} \|\sum_{i \in S_1}\nabla F_i(\theta^*_1) \|  \\
    & \leq \frac{\gamma}{n} \sum_{i \in S_1} \|\nabla F_i(\theta^*_1) \| \\
    & = \frac{\gamma}{n}   \sum_{i \in S_1 \cap S^*_1} \|\nabla F_i(\theta^*_1) \| + \frac{\gamma}{n} \sum_{j=2}^{k} \sum_{i \in S_1 \cap S^*_j} \|\nabla F_i(\theta^*_1) \|
\end{align*}

For $i \in S^*_1$, the term $\|\nabla F_i(\theta^*_1) \| \leq \mu$. With this, we have
\begin{align*}
    T_2 &\leq \frac{\gamma}{n} |S_1 \cap S^*_1| \mu + \frac{\gamma}{n} \sum_{j=2}^k |S_1 \cap S^*_j| \max_{i} \|\nabla F_i(\theta^*_1) \| \\
    & \leq \frac{\gamma}{n} |S^*_1| \mu \|\theta^*_1\| + \frac{c_1 \gamma}{n} (k-1) \exp \left( - c_2 \frac{\Delta_1^2}{\rho^2}\, d \right) n \max_{i} \|\nabla F_i(\theta^*_1) \|,
\end{align*}
where $\rho = (\max_{j \in [k]} \|\theta^*_j\|)^2$ and $\Delta_1 = \Delta - 2\lambda$. Here, we use Lemma~\ref{lem:error} in conjunction with sub-Gaussian concentration. Finally, we have
\begin{align*}
    T_2 &\leq \gamma p_1 \mu + C_1 k \gamma \exp \left( - c_2 \frac{\Delta_1^2}{\rho^2}\, d \right) ( \max_i \|x_i\|^2 \|\theta^*_1\| + \max_i \|x_i\| |y_i|) \\
    & \leq \gamma p_1 \mu \|\theta^*_1\| + c_1 k \gamma \exp \left( - c_2 \frac{\Delta_1^2}{\rho^2}\, d \right)  \|\theta^*_1\|,
\end{align*}
where we use that $\|x_i\| \leq 1$ for all $i$, and $|y_i| \leq 1$ for all $i$.

Combining $T_1$ and $T_2$, we have
\begin{align*}
    \|\thplusone - \theta^*_1\| \leq (1-c\gamma p_1) \|\theta_1 - \theta^*_1\| + \gamma p_1 \mu \|\theta^*_1\| + c_1 k \gamma \exp \left( - c_2 \frac{\Delta_1^2}{\rho^2}\, d \right)  \|\theta^*_1\|
\end{align*}
with high probability. Choosing $\gamma = \frac{1}{2cp_1}$, we have
\begin{align*}
    \|\thplusone - \theta^*_1\| \leq \frac{1}{2} \|\theta_1 - \theta^*_1\| + c_1 \mu \|\theta^*_1\| +  c_2 k \gamma \exp \left( - c_3 \frac{\Delta^2}{\rho^2}\, d \right)  \|\theta^*_1\|,
\end{align*}
with probability at least $1-C_1 \exp(-C_2 d)$, which proves the theorem.

\subsubsection{Good Initialization}
We stick to analyzing $\thplusone$. In the following lemma, we only consider $\theta_2$. In general, for $\{\theta_3, \ldots, \theta_k\}$, the total probability of error, using union bound will be $(k-1)$ times the probability of error obtained in the following lemma.
\begin{lemma}
\label{lem:error}
We have
\begin{align*}
    \PP \bigg ( F_i (\theta_1) > F_i(\theta_2)|F_i(\theta^*_1) < F_i (\theta^*_2) \bigg ) \leq c_1 k\exp \left( - c_2 \frac{\Delta_1^2}{\rho^2}\, d \right),
\end{align*}
where $\rho = (\max_{j \in [k]} \|\theta^*_j\|)^2$ and $\Delta_1 = \Delta - 2\lambda$.
\end{lemma}

\begin{proof}
\noindent Let us first consider the event 
\begin{align*}
  \left \lbrace F_i(\theta^*_1) < F_i (\theta^*_2) \right \rbrace = \left \lbrace (y_i - \inprod{x_i}{\theta^*_1})^2 < (y_i - \inprod{x_i}{\theta^*_2})^2 \right \rbrace
\end{align*}
We are interested in the event $\{F_i(\theta_1) < F_i(\theta_2)\}$ since the sets $S_1$ and $S_2$ are defined based on this event. In particular we are interested in $\PP(F_i(\theta_1) < F_i(\theta_2)|\left \lbrace F_i(\theta^*_1) < F_i (\theta^*_2) \right \rbrace)$. To that end, we first compute the complement event, given by
\begin{align*}
    \PP(F_i(\theta_1)> F_i(\theta_2) |\left \lbrace F_i(\theta^*_1) < F_i (\theta^*_2) \right \rbrace ),
\end{align*}

Conditioned on $\left \lbrace F_i(\theta^*_1) < F_i (\theta^*_2) \right \rbrace$, we compute the probability of the event
\begin{align*}
 & \{  (y_i - \inprod{x_i}{\theta_1})^2 > (y_i - \inprod{x_i}{\theta_2})^2 \} \\
  &= \{ (y_i - \inprod{x_i}{\theta^*_1})^2 + \inprod{x_i}{\theta_1 -\theta^*_1}^2 - 2\inprod{x_i}{\theta_1-\theta^*_1}(y_i - \inprod{x_i}{\theta^*_1} \\
  & \qquad  >  (y_i - \inprod{x_i}{\theta^*_2})^2 + \inprod{x_i}{\theta_2 -\theta^*_2}^2 - 2\inprod{x_i}{\theta_2-\theta^*_2}(y_i - \inprod{x_i}{\theta^*_2} \} \\
  & \subseteq \{ \inprod{x_i}{\theta_1 -\theta^*_1}^2 - 2\inprod{x_i}{\theta_1-\theta^*_1}(y_i - \inprod{x_i}{\theta^*_1}) \\
   & \qquad  >  \inprod{x_i}{\theta_2 -\theta^*_2}^2 - 2\inprod{x_i}{\theta_2-\theta^*_2}(y_i - \inprod{x_i}{\theta^*_2}) \},
\end{align*}
where we use the fact that $(y_i - \inprod{x_i}{\theta^*_1})^2 < (y_i - \inprod{x_i}{\theta^*_2})^2$. Continuing, we have
\begin{align*}
    & \{  (y_i - \inprod{x_i}{\theta_1})^2 > (y_i - \inprod{x_i}{\theta_2})^2 \} \\
     & \subseteq \{ \inprod{x_i}{\theta_1 -\theta^*_1}^2 - 2\inprod{x_i}{\theta_1-\theta^*_1}(y_i - \inprod{x_i}{\theta^*_1}) > t \} \bigcup \{ \inprod{x_i}{\theta_2 -\theta^*_2}^2 - 2\inprod{x_i}{\theta_2-\theta^*_2}(y_i - \inprod{x_i}{\theta^*_2}) \leq t \},
\end{align*}
for any $t$. In particular, we select $t= |\inprod{x_i}{\theta_1 - \theta^*_1}|\,\,| \inprod{x_i}{\theta_2 - \theta^*_2}|$. We have
\begin{align*}
    & \inprod{x_i}{\theta_1 -\theta^*_1}^2 - 2\inprod{x_i}{\theta_1-\theta^*_1}(y_i - \inprod{x_i}{\theta^*_1}) \\
    & = \inprod{x_i}{\theta_1 -\theta^*_1}^2 + 2\inprod{x_i}{\theta^*_1-\theta_1}(y_i - \inprod{x_i}{\theta^*_1}) \\
    & \leq |\inprod{x_i}{\theta_1 -\theta^*_1}|^2 + 2 |\inprod{x_i}{\theta^*_1 -\theta_1}| \, |y_i - \inprod{x_i}{\theta^*_1}| \\
    & = |\inprod{x_i}{\theta_1 -\theta^*_1}|^2 + 2 |\inprod{x_i}{\theta_1 -\theta^*_1}| \, |y_i - \inprod{x_i}{\theta^*_1}|
\end{align*}
and
\begin{align*}
    & \inprod{x_i}{\theta_2 -\theta^*_2}^2 - 2\inprod{x_i}{\theta_2 -\theta^*_2}(y_i - \inprod{x_i}{\theta^*_2}) \\
    & \geq |\inprod{x_i}{\theta_2 -\theta^*_2}|^2 - 2 |\inprod{x_i}{\theta_2 -\theta^*_2}| \, |y_i - \inprod{x_i}{\theta^*_2}|
\end{align*}
With this we have,
\begin{align*}
    & \mathbb{P} \left((y_i - \inprod{x_i}{\theta_1})^2 > (y_i - \inprod{x_i}{\theta_2})^2 | \mathcal{E}_i \right) \\
    & \leq \underbrace{\mathbb{P} \left (|\inprod{x_i}{\theta_1 - \theta^*_1}| > |\inprod{x_i}{\theta_2 -\theta^*_2}| - 2 |y_i - \inprod{x_i}{\theta^*_1}| \right)}_{T_1} \\
    & + \underbrace{\mathbb{P} \left (|\inprod{x_i}{\theta_2 - \theta^*_2}| < |\inprod{x_i}{\theta_1 -\theta^*_1}| + 2 |y_i - \inprod{x_i}{\theta^*_1}| \right)}_{T_2}.
\end{align*}
We first analyze $T_1$. The analysis for the $T_2$ term follows similarly. We have
\begin{align}
    T_1 &= \mathbb{P} \left (|\inprod{x_i}{\theta_1 - \theta^*_1}| > |\inprod{x_i}{\theta_2 -\theta^*_2}| - 2 |y_i - \inprod{x_i}{\theta^*_1}| \right) \\
    & \leq \mathbb{P} \left( |\inprod{x_i}{\theta_2 -\theta^*_2}| > \Delta \right) + \mathbb{P} \left( |\inprod{x_i}{\theta_1 - \theta^*_1}| > \Delta - 2 \lambda \right),
\end{align}
where $\Delta$ is the separation as defined in Section~\ref{sec:psetting} and $\lambda$ is a small parameter defined in.. With this we have
\begin{align*}
   T_1 = &\PP \left(y_i - \inprod{x_i}{\theta_1})^2 > (y_i - \inprod{x_i}{\theta_2})^2 | \mathcal{E}_i \right) \\
    & = \PP \left ( \inprod{x_i}{\frac{\theta_2-\theta^*_2}{\|\theta_2-\theta^*_2\|}} >  \, \frac{\Delta}{\|\theta_2-\theta^*_2\|} \right) + \PP \left ( \inprod{x_i}{\frac{\theta_1-\theta^*_1}{\|\theta_1-\theta^*_1\|}} >  \, \frac{\Delta - 2\lambda}{\|\theta_1-\theta^*_1\|} \right)
\end{align*}
Now, note that $\theta_2 - \theta^*_2$ and $\theta_1 - \theta^*_1$ depend on $x_i$; since we are not allowed to have fresh samples at every iteration.

To avoid such dependence we compute an uniform bound using one-step discretization (covering argument). Suppose $u_1,u_2,\ldots,u_n$ be an $\epsilon$ cover of $ \mathbb{S}^{d-1}$ in $\ell_2$ norm. \cite{vershynin2018high} shows that $M \leq (1+\frac{2}{\epsilon})^d$. Also, suppose $u_l$ is the nearest point of the $\epsilon$ cover to $\frac{\theta_2-\theta^*_2}{\|\theta_2-\theta^*_2\|}$ and $u_r$ is the nearest point of the $\epsilon$ cover to $\frac{\theta_1-\theta^*_1}{\|\theta_1-\theta^*_1\|}$. We have
\begin{align*}
   & = \PP \left ( \inprod{x_i}{\frac{\theta_2-\theta^*_2}{\|\theta_2-\theta^*_2\|}} >  \, \frac{\Delta}{\|\theta_2-\theta^*_2\|} \right) + \PP \left ( \inprod{x_i}{\frac{\theta_1-\theta^*_1}{\|\theta_1-\theta^*_1\|}} >  \, \frac{\Delta - 2\lambda}{\|\theta_1-\theta^*_1\|} \right) \\
    & =\PP \left ( \inprod{x_i}{u_l} +  \inprod{x_i}{\frac{\theta_2-\theta^*_2}{\|\theta_2-\theta^*_2\|}-u_l} >  \, \frac{\Delta}{\|\theta_2-\theta^*_2\|} \right) \\
    & \qquad + \PP \left ( \inprod{x_i}{u_r} +  \inprod{x_i}{\frac{\theta_1-\theta^*_1}{\|\theta_1-\theta^*_1\|}-u_r} >  \, \frac{\Delta -2\lambda}{\|\theta_1-\theta^*_1\|} \right) \\
    & =\PP \left ( \inprod{x_i}{u_l} >  \, \frac{\Delta}{\|\theta_2-\theta^*_2\|} -\epsilon \right) + \PP \left ( \inprod{x_i}{u_r}  >  \, \frac{\Delta -2\lambda}{\|\theta_1-\theta^*_1\|} - \epsilon \right) \\
    & \leq \PP \left ( \inprod{x_i}{u_l} >  \, \frac{\Delta\,\sqrt{d}}{\Tilde{c} \|\theta^*_2\|} - \epsilon \right) +  \PP \left ( \inprod{x_i}{u_r} >  \, \frac{(\Delta -2\lambda)\,\sqrt{d}}{\Tilde{c}_1 \|\theta^*_1\|} - \epsilon \right)
\end{align*}
where we use the fact that $\|x_i\| \leq 1$, the definition of an $\epsilon$ cover and the initialization condition.

We now use the fact that, $x_i$ is centered sub-Gaussian (with constant parameter) random variables.  Note that since 
\begin{align*}
    \mathbb{P}(F_i(\theta^*_1)  < F_i (\theta^*_2) = p_1 \geq \min \{p_1,\ldots,p_k \} > \Bar{c},
\end{align*}
where $\Bar{c}$ is a constant, we invoke \citep[Lemma 15]{yi2016solving} to obtain
\begin{align*}
    \PP \left ( \inprod{x_i}{u_l} >  \frac{\Delta\,\sqrt{d}}{\Tilde{c} \|\theta^*_2\|} - \epsilon \right) \leq \exp \left ( - c (  \frac{\Delta\,\sqrt{d}}{\Tilde{c} \|\theta^*_2\|} - \epsilon)^2 \right),
\end{align*}
and
\begin{align*}
    \PP \left ( \inprod{x_i}{u_r} >  \frac{(\Delta -2\lambda)\,\sqrt{d}}{\Tilde{c}_1 \|\theta^*_1\|} - \epsilon \right) \leq \exp \left ( - c  (  \frac{(\Delta -2\lambda)\,\sqrt{d}}{\Tilde{c} \|\theta^*_1\|} - \epsilon)^2 \right),
\end{align*}
We now take union bound over $(1+\frac{2}{\epsilon})^d$ entries, and finally substituting $\epsilon$ as a sufficiently small constant, we have
\begin{align*}
    \PP \left ( \inprod{x_i}{\frac{\theta_2-\theta^*_2}{\|\theta_2-\theta^*_2\|}} >  \frac{\Delta}{\|\theta_2-\theta^*_2\|} \right) \leq c_1 \exp \left( - c_2 \frac{\Delta^2}{\|\theta^*_2\|^2}\, d \right),
\end{align*}
and 
\begin{align*}
    \PP \left ( \inprod{x_i}{\frac{\theta_1-\theta^*_1}{\|\theta_1-\theta^*_1\|}} >  \frac{\Delta-2\lambda}{\|\theta_1-\theta^*_1\|} \right) \leq c_1 \exp \left( - c_2 \frac{(\Delta-2\lambda)^2}{\|\theta^*_1\|^2}\, d \right).
\end{align*}
Note that similar analysis holds for the term $T_2$, since we use the symmetric sub-Gaussian concentration. Hence, we finally obtain,
\begin{align*}
    \PP \bigg ( F_i (\theta_1) > F_i(\theta_2)|F_i(\theta^*_1) < F_i (\theta^*_2) \bigg ) \leq C \exp \left( - c \frac{(\Delta-2\lambda)^2}{(\max_{j \in [k]} \|\theta^*_j\|)^2}\, d \right),
\end{align*}
and
\begin{align*}
    \PP \bigg ( F_i (\theta_1) < F_i(\theta_2)|F_i(\theta^*_1) < F_i (\theta^*_2) \bigg ) \geq 1 -C \exp \left( - c \frac{(\Delta-2\lambda)^2}{(\max_{j \in [k]} \|\theta^*_j\|)^2}\, d \right),
\end{align*}
which concludes the proof.
\end{proof}
\begin{lemma}
\label{lem:subgauss}
Conditioned on $i \in S_1 \cap S^*_1$, the distribution of $x_i$ is sub-Gaussian with a constant parameter.
\end{lemma}
\begin{proof}
Let us first consider the following probability
\begin{align*}
   & \PP(i \in S_1 \cap S^*_1) \\
   & = \PP \left( F_i(\theta_1) < F_i(\theta_2), F_i(\theta^*_1) <F_i(\theta^*_2) \right) \\
    & =  \PP \left( F_i(\theta_1) < F_i(\theta_2)| F_i(\theta^*_1) <F_i(\theta^*_2) \right) \PP \left(  F_i(\theta^*_1) <F_i(\theta^*_2) \right) \\
    & \geq p_1 \left[1 -C \exp \left( - c \frac{(\Delta-2\lambda)^2}{(\max_{j \in [k]} \|\theta^*_j\|)^2}\, d \right) \right] \\
    & \geq \tau
\end{align*}
where $\tau$ is a constant. Here we use Assumption~\ref{asm:bounded} and the fact that $\min\{p_1,p_2\} > \Bar{c}$ where $\Bar{c}$ is a constant. Since $\PP(i \in S_1 \cap S^*_1)$ is greater than a constant, using \citep[Lemma 15]{yi2016solving}, we conclude that $x_i$ is sub-Gaussian with constant parameter.
\end{proof}

\section{Proofs from Section~\ref{sec:gen}}
\label{sec:proof_gen}

We now provide the proof of Theorem~\ref{thm:gen}.

\begin{proof}[Proof of Theorem~\ref{thm:gen}]
We begin by the following chain,

\begin{align*}
    \frac{1}{n}\EE_{\bsigma} \left[ \sup_{\bar{h} \in \bar{\cH}} \sum_{i=1}^{n} \sigma_i \cL(y_i, \bar{h}(x))\right] = \frac{1}{n} \EE_{\sigma_1, \cdots, \sigma_{n-1}}\left[ \EE_{\sigma_n} \left[ \sup_{\bar{h} \in \bar{\cH}} u_{n-1}(\bar{h}) + \sigma_n \cL(y_n, \bar{h}(x_n)) \right]\right],
\end{align*}
where $u_n(\bar{h}) = \sum_{i=1}^{n-1} \sigma_i \cL(y_i, \bar{h}(x_i))$. By definition of the supremum, for any $\epsilon > 0$, there exist $\bar{h}_a$ and $\bar{h}_b$ such that,
\begin{align*}
    u_{n-1}(\bar{h}_a) + \cL(y_n, \bar{h}_a(x_n)) &\geq (1 - \epsilon) \left[\sup_{\bar{h}} u_{n-1}(\bar{h}) + \cL(y_n, \bar{h}(x_n))\right] \\
    u_{n-1}(\bar{h}_b) - \cL(y_n, \bar{h}_b(x_n)) &\geq (1 - \epsilon) \left[\sup_{\bar{h}} u_{n-1}(\bar{h}) - \cL(y_n, \bar{h}(x_n))\right]
\end{align*}
By expanding the definition of the expectation we can then prove that,
\begin{align*}
    (1 - \epsilon) \EE_{\sigma_n} \left[ \sup_{\bar{h}} u_{n-1}(\bar{h}) + \sigma_n\cL(y_n, \bar{h}(x_n)) \right] \leq \frac{1}{2} \left( u_{n-1}(\bar{h}_a) + \cL(y_n, \bar{h}_a(x_n)) \right) + \frac{1}{2} \left( u_{n-1}(\bar{h}_b) - \cL(y_n, \bar{h}_b(x_n)) \right).
\end{align*}

Let $s_l = \mathrm{sign}(\bar{h}_a(x_n)_l - \bar{h}_b(x_n)_l)$. Then we have the following chain,
\begin{align*}
    &(1 - \epsilon) \EE_{\sigma_n} \left[ \sup_{\bar{h}} u_{n-1}(\bar{h}) + \sigma_n\cL(y_n, \bar{h}(x_n)) \right] \\
    &\leq \frac{1}{2} \left(u_{n-1}(\bar{h}_a) + u_{n-1}(\bar{h}_b) + \cL(y_n, \bar{h}_a(x_n)) - \cL(y_n, \bar{h}_b(x_n))  \right) \\
    &= \frac{1}{2} \left(u_{n-1}(\bar{h}_a) + u_{n-1}(\bar{h}_b) + \min_{j} \ell(y_n, \bar{h}_a(x_n)_j) - \min_{l} \ell(y_n, \bar{h}_b(x_n)_l)  \right) \\
    &\stackrel{(a)}{\leq} \frac{1}{2} \left(u_{n-1}(\bar{h}_a) + u_{n-1}(\bar{h}_b) + \ell(y_n, \bar{h}_a(x_n)_{l^*}) -  \ell(y_n, \bar{h}_b(x_n)_{l^*})  \right) \\
    &\leq \frac{1}{2} \left(u_{n-1}(\bar{h}_a) + u_{n-1}(\bar{h}_b) + \max_l (\ell(y_n, \bar{h}_a(x_n)_{l}) -  \ell(y_n, \bar{h}_b(x_n)_{l})) \right) \\
    &\leq \frac{1}{2} \left(u_{n-1}(\bar{h}_a) + u_{n-1}(\bar{h}_b) + \sum_l (\ell(y_n, \bar{h}_a(x_n)_{l}) -  \ell(y_n, \bar{h}_b(x_n)_{l})) \right) \\
    &\leq \frac{1}{2} \left(u_{n-1}(\bar{h}_a) + u_{n-1}(\bar{h}_b) + \sum_l \mu s_l(\bar{h}_a(x_n)_{l} -  \bar{h}_b(x_n)_{l}) \right) \\
    &= \sum_{l=1}^{k} \frac{1}{2} \left(u_{n-1}(\bar{h}_a)/k + \mu s_l\bar{h}_a(x_n)_l \right) + \frac{1}{2} \left(u_{n-1}(\bar{h}_b)/k - \mu s_l\bar{h}_b(x_n)_l \right) \\
    &\leq \sum_{l} \EE_{\sigma_{n,l}} \left[\sup_{\bar{h} \in \bar{\cH}} \frac{u_{n-1}(\bar{h})}{k} + \sigma_{n,l}\bar{h}(x_n)_l \right] \\
    &= k \EE_{\sigma_{n}} \left[\sup_{\bar{h} \in \bar{\cH}} \frac{u_{n-1}(\bar{h})}{k} + \mu\sigma_{n}h_j(x_n) \right].
\end{align*}
where $\sigma_{n,l}$ are Rademacher R.Vs. The result follows by induction and symmetry between the mixture coordinates of $\bar{h}(\cdot)$. In $(a)$, $l^*$ is the coordinate at which the first minima is attained. 
\end{proof}

\section{Proofs from Section~\ref{sec:subset}}
\label{sec:proof_subset}

We first start by proving that for any of the optimal subsets of the the original samples if we further sub-sample a smaller partition, then the regression solution in both the sets are not too far from each other. The result follows from an application of random design regression analysis in~\citep{hsu2011analysis}.

\begin{lemma}
\label{lem:regk}
For any $i \in [k]$, let $B = B^*_i$. Suppose $P$ samples are chosen i.i.d uniformly at random from $B$. Let $\theta_B$ and $\theta_P$ be the least squares solution for $B$ and $P$ respectively. Then we have that,
\begin{align*}
    \norm{\theta_B - \theta_P}_{\Sigma_B}^2 = O \left( \frac{R(\sqrt{d} + \sqrt{\log(1/\delta)})^2 B_d }{|P|\sqrt{\lmin(\Sigma_B)}}  + \frac{4 (b + wR)^2 (\sqrt{d} + \sqrt{\log(1/\delta)})^2}{|P|} + \frac{\gamma^2 d^2 \log^2(1/\delta)}{|P|^2}\right),
\end{align*}
with probability at least $1 - \delta$.
\end{lemma}

\begin{proof}
 The samples in $P$ form a random design regression problem where,

\begin{align*}
    \EE[xy] &= \frac{1}{|B|}\sum_{i \in B} x_iy_i \\
    \EE[xx^T] &=  \frac{1}{|B|}\sum_{i \in B} x_ix_i^T = : \Sigma_B.
\end{align*}

We will now verify the conditions of Theorem~2 in~\cite{hsu2011analysis}. To avoid cumbersome notation let $\theta = \theta_B$.
Consider the decomposition of the regression target and covariates,
\begin{align*}
    y = \theta^Tx + \mathrm{bias}(x) + \eta(x)
\end{align*}
where $\eta(x)$ is the noise. The terms are defined as,

\begin{align*}
    \bias(x) = E[y | x] - \theta^Tx \text{  and  } \eta(x) = y - E[y | x]
\end{align*}

\emph{Condition 1}: Note that the noise is bounded,
\begin{align*}
    |\eta(x)| \leq |y| + |\theta^Tx| + |\mathrm{bias}(x)| \leq 2(b + wR).
\end{align*}
Therefore the noise is sub-gaussian with parameter $\sigma_{noise} \leq 2(b + wR)$.

\emph{Condition 2:} The bounded bias is just from our assumption directly. We have the following,
\begin{align}
    \norm{\Sigma_B^{-1} \bias(x) x} \leq  \gamma\sqrt{d}.
\end{align}

\emph{Condition 3:} For any $u$ s.t $\norm{u} = 1$, we have that
\begin{align}
    u^T \Sigma _B^{-1/2} x \leq \frac{R}{\sqrt{\lambda_{min}(\Sigma _B)}}.
\end{align}
Thus, $\rho_{(1, cov)} = \frac{R}{\sqrt{\lambda_{min}(\Sigma_B)}}$. Thus all the conditions are valid and we can apply Theorem 2 in~\cite{hsu2011analysis}.  Thus finally with probability $1 - \delta$ we get,

\begin{align*}
    \norm{\theta_P - \theta_B}_{\Sigma_B}^2 = O \left( \frac{\rho_{(1,cov)}(\sqrt{d} + \sqrt{\log(1/\delta)})^2 \EE[\bias(x)^2]}{|P|}  + \frac{\sigma^2_{noise} (\sqrt{d} + \sqrt{\log(1/\delta)})^2}{|P|} + \frac{\gamma^2 d^2 \log^2(1/\delta)}{|P|^2}\right).
\end{align*}
\end{proof}

Now we show that the above result also implies that the respective average errors are also close to each other.

\begin{lemma}
\label{lem:error}
Consider the setting in Lemma~\ref{lem:regk}. Then we have the following bound,
\begin{align*}
    \cE_{ls}(B, \theta_P) -  \cE_{ls}(B, \theta_B) \leq 2|B|(wR + b) \norm{\theta_B - \theta_P}_{\Sigma_B}.
\end{align*}
\end{lemma}

\begin{proof}
We should note that,
\begin{align*}
    &\sum_{i \in B} (y_i - \theta_P x_i)^2 - (y_i - \theta_B x_i)^2 \leq |(\theta_P - \theta_B).x_i| |(2y_i - \theta_B.x_i - \theta_P.x_i)| \\
    &:= \gamma_i|(\theta_P - \theta_B).x_i| \leq \sqrt{\left(\sum \gamma_i^2\right)\left( \sum ((\theta_P - \theta_B).x_i)^2\right)} \leq 2|B|(wR + b) \norm{\theta_B - \theta_P}_{\Sigma_B}. 
\end{align*}
\end{proof}

Now we are at a position to prove our main theorem.

\begin{proof}[Proof of Theorem~\ref{thm:subsamp}]
Let $\{\tilde{P}_i\}$ be the partition of $A$ that conforms to the partition $\{B_i^*\}$. Then by Hoeffding's bound we have w.p $1 - \delta$
\begin{align*}
    ||\tilde{P}_i| - \alpha_i |A|| \leq \sqrt{\alpha_i(1 - \alpha_i) \log (k/\delta)|A|}.
\end{align*}

Let us assume that $n$ is such that $|A| \geq \frac{4(1 - \alpha_i)}{\alpha_i} \log (k/ \delta)$, for all $i$.

Then by Lemma~\ref{lem:error} we have that,
\begin{align*}
   \sum_{i=1}^{k} \cE_{ls}(B_i^*, \theta_{\tilde{P}_i}) -  \cE_{ls}(B_i^*, \theta_{B_i^*}) &\leq 2 (wR + b) \sum_{i=1}^{k} |B_i^*|\norm{\theta_{B_i^*} - \theta_{\tilde{P}_i}}_{\Sigma_{B_i^*}}.
\end{align*}

Thus w.p at least $1 - 2\delta$ we have,
\begin{align*}
    \frac{1}{n} \sum_{i=1}^{k}\cE_{ls}(B_i^*, \theta_{\tilde{P}_i}) &\leq  \frac{1}{n} \sum_{i=1}^{k}\cE_{ls}(B_i^*, \theta_{B_i^*}) \\
    & + O \left(\sum_{i=1}^{k} \alpha_i (wR + b) \sqrt{\frac{R(\sqrt{d} + \sqrt{\log(k/\delta)})^2 B_d }{\alpha_i|A|\sqrt{\lmin}}  + \frac{(wR + B)^2 (\sqrt{d} + \sqrt{\log(k/\delta)})^2}{\alpha_i|A|}}\right) 
\end{align*}

Now note that $\cE(\{P_i^m\})$ in Algorithm~\ref{algo:subset} is less than the LHS in the above equation. Therefore, the loss achieved by the solution of Algorithm~\ref{algo:subset} can only be less than the RHS of the above equation.
\end{proof}

The above result means that an $\epsilon$ additive approximation is obtained when,
\begin{align*}
    |A| = O \left(\frac{1}{\epsilon^2} k^2 \frac{\alpha}{\lmin} \left(d + \log \frac{k}{\delta} \right) \right)
\end{align*}
w.p at least $1 - 2\delta$.

\section{More experiments}

\begin{figure}[tb]
\centering
    \includegraphics[scale = 0.5]{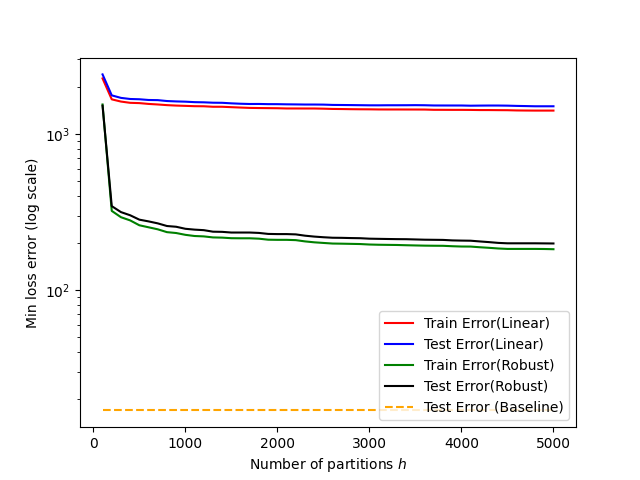}
    \caption{Comparison of train error and test error (min-loss) of Algorithms \ref{algo:subset2}  (with number of random partitions $h$) when we use a linear model/robust linear model to fit each part of partition.\label{fig:compare}}
\end{figure}

\paragraph{Mixture of linear datasets:} In this experiment, we create two datasets $(X_1,y_1)$ (using \texttt{sklearn.datasets.makeregression} with $4$ features, standard deviation of noise = $4.0$ and bias = $100.0$) and $(X_2,y_2)$ (using \texttt{makeregression} module from \texttt{sklearn.datasets} with $4$ features, standard deviation of noise = $4.0$ and bias = $0.0$) generated according to two distinct noisy linear models. Subsequently we combine the datasets and shuffle the datapoints to create a single one $(X,y)$. We partition the dataset $(X,y)$ intro training data ($3200$ datapoints) and test data ($800$ datapoints). We first design a baseline model as if an oracle provides us the information of which datapoint in $(X,y)$ was generated according to which mode. In particular we fit two distinct linear regression models $\beta_1,\beta_2$ on the training data restricted to $(X_1,y_1)$ and $(X_2,y_2)$; subsequently, we predict on the test data using the two linear regression models and compute the min-loss. Using this model, the mean-squared training error is $15.256$ and the mean min-loss error on the test data is $17.013$. Next, we compare the performance (min-loss on the train and test data) of Algorithm $\ref{algo:subset2}$ (averaged over $100$ implementations and $|A|=150$) which is depicted in Figure \ref{fig:compare}.
In Figure \ref{fig:compare}, we empirically compare the two possibilities namely the linear model and the robust linear model on each part of the partition. It is empirically quite evident that Algorithm \ref{algo:subset2} when used with the robust regression model achieves a significantly better performance as compared to Algorithm \ref{algo:subset2} used with the linear model.
We also implement the AM algorithm (see Algorithm \ref{alg:am}) on this dataset. Here, we initialize $\theta_1^{(0)},\theta_1^{(1)}$ by $\beta_1+\eta_1$ and $\beta_2+\eta_2$ where $\eta_1,\eta_2$ are random vectors where each entry is generated independently according to a Gaussian with zero mean and variance $\sigma^2$. We vary $\sigma$ and after $40$ iterations, we compare the min-loss error (averaged over $15$ implementations with $\gamma=0.1$) on the train and test datasets. We provide these results in Table \ref{table:am}. 

\begin{table}[t]

\begin{center}
\begin{tabular}{ c|c|c|c|c } 
 \toprule
 $\sigma$ & 10 & 50 & 100 & 200 \\ \midrule 
 Train Error & 15.65 & 96.44 & 1350.22 & 2408.80  \\ 
 Test Error & 15.35 & 94.35 & 1341.29 & 2380.27  \\ 
 \bottomrule
\end{tabular}

\end{center}
\caption{Train/Test error (Min loss) of Algorithm \ref{alg:am} for different initializations after $40$ iterations}
\label{table:am}
\end{table}

{\bf More on other experiments: } We provide some missing figures and other details from the experiments here. Table~\ref{table:nonlinear1} provides the training set min-losses for all the non-linear synthetic datasets. Table~\ref{table:samples} provides the datasets statistics for the movie lens dataset that we use.

\begin{table}[t]

\begin{center}
\begin{tabular}{ c|c|c|c } 
 \toprule
 Dataset & LR  & Alg \ref{alg:am} (Mean, Var) & Alg \ref{algo:subset2} (Mean,Var) \\ \midrule
   A &     22.57  &  (20.37,10.02) & (12.01,0.21)  \\
   B &     18806  & N/A & (5105, 51537) \\
   C &     15.34  & N/A & (7.24,0.16)  \\    
 \bottomrule
\end{tabular}

\end{center}
\caption{Mean (Mean) and Variance (Var) of min-loss on Training data generated from non-linear synthetic datasets. LR corresponds to a simple linear regression model and N/A implies that the algorithm did not converge in any implementation.}
\label{table:nonlinear1}
\end{table}

\begin{table}[t]

\begin{center}
\begin{tabular}{ c|c|c|c } 
 \toprule
 $u$ & $v$ & Train Data & Test Data \\ \midrule
 `1010' & `2116' & 836 & 209 \\
 `752' & `1941' & 1002 & 251 \\
 `752' &  `2116' & 827 & 207  \\
 `752' & `2909' & 844 & 212\\
 `1010' & `4725' & 818 & 205 \\
 \bottomrule
\end{tabular}

\end{center}
\caption{Size of train and test data for users $(u,v)$}
\label{table:samples}
\end{table}

\end{document}